\documentclass[11pt]{article}
\usepackage{times,fullpage}
\usepackage[numbers]{natbib}
\usepackage{bbm,amsmath,amsfonts, amsopn, amssymb,amsthm}
\usepackage{color}
\usepackage{graphicx, epstopdf}
\usepackage{mathtools}

\newtheorem{theorem}{Theorem}[section]

\newtheorem{corollary}[theorem]{Corollary}
\newtheorem{proposition}[theorem]{Proposition}
\newtheorem{definition}[theorem]{Definition}

\theoremstyle{remark}
\newtheorem{remark}[theorem]{Remark}

\newcommand{\BALD}{\begin{aligned}}
\newcommand{\EALD}{\end{aligned}}
\newcommand{\BALDS}{\begin{aligned*}}
\newcommand{\EALDS}{\end{aligned*}}
\newcommand{\BCAS}{\begin{cases}}
\newcommand{\ECAS}{\end{cases}}
\newcommand{\BEAS}{\begin{eqnarray*}}
\newcommand{\EEAS}{\end{eqnarray*}}
\newcommand{\BEQ}{\begin{equation}}
\newcommand{\EEQ}{\end{equation}}
\newcommand{\BIT}{\begin{itemize}}
\newcommand{\EIT}{\end{itemize}}
\newcommand{\BMAT}{\begin{bmatrix}}
\newcommand{\EMAT}{\end{bmatrix}}
\newcommand{\BNUM}{\begin{enumerate}}
\newcommand{\ENUM}{\end{enumerate}}

\newcommand{\BA}{\begin{array}}
\newcommand{\EA}{\end{array}}

\newcommand{\reals}{\mathbf{R}}

\newcommand{\diag}{\mathop{\mathbf{diag}}}

\DeclareMathOperator{\linspan}{span}
\newcommand{\norm}[1]{\left\| #1 \right\|}

\newcommand{\R}{\mathbb{R}}

\title{Gradient Descent Converges to Minimizers}

\author{Jason D. Lee$^\sharp,$ Max Simchowitz$^\sharp$, Michael I. Jordan$^\sharp\dagger$,
and Benjamin Recht$^\sharp\dagger$\\
$^\sharp$Department of Electrical Engineering and Computer Sciences\\
$^\dagger$Department of Statistcs\\
University of California, Berkeley
}

\begin{document}

\maketitle

\begin{abstract}
We show that gradient descent converges to a local minimizer, almost surely with random initialization.  This is proved by applying the Stable Manifold Theorem from dynamical systems theory.
\end{abstract}

\section{Introduction}

Saddle points have long been regarded as a tremendous obstacle for continuous optimization. There are many well known examples when worst case initialization of gradient descent provably converge to saddle points \citep[Section 1.2.3]{nesterov2004introductory}, and hardness results which show that finding even a \emph{local} minimizer of non-convex functions is NP-Hard in the worst case~\citep{murty1987some}.  However, such worst-case analyses have not daunted practitioners, and high quality solutions of continuous optimization problems are readily found by a variety of simple algorithms.  Building on tools from the theory of dynamical systems, this paper demonstrates that, under very mild regularity conditions, saddle points are indeed of little concern for the gradient method.

More precisely, let $f: \R^d \to \R$ be twice continuously differentiable, and consider the classic gradient method with constant step size $\alpha$:
\begin{equation}\label{GradDescent}
x_{k+1} =x_k - \alpha \nabla f(x_k).
\end{equation}
We call $x$ a critical point of $f$ if $\nabla f(x) = 0$, and say that $f$ satisfies the $\emph{strict saddle property}$ if each critical point $x$ of $f$ is either a local minimizer, or a ``strict saddle'', i.e, $\nabla^2 f(x)$ has at least one strictly negative eigenvalue. We prove:
\begin{quote}
If $f:\R^d \to \R$ is twice continuously differentiable and satisfies the strict saddle property, then gradient descent (Equation~\ref{GradDescent}) with a random initialization and sufficiently small constant step size converges to a local minimizer or negative infinity almost surely.
\end{quote}
Here, by sufficiently small, we simply mean less than the inverse of the Lipschitz constant of the gradient. As we discuss below, such step sizes are standard for the gradient method.  We remark that the strict saddle assumption is necessary in the worst case, due to hardness results regarding testing the local optimality of functions whose Hessians are highly degenerate at critical points (e.g, quartic polynomials)~\citep{murty1987some}.

\subsection{Related work}

Prior work has show that first-order descent methods can circumvent strict saddle points, provided that they are augmented with unbiased noise whose variance is sufficiently large along each direction. For example, \cite{pemantle1990nonconvergence} establishes convergence of the Robbins-Monro stochastic approximation to local minimizers for strict saddle functions. More recently, \cite{ge2015escaping} give quantitative rates on the convergence of noise-added stochastic gradient descent to local minimizers,  for strict saddle functions. The condition that the noise have large variance along all directions is often not satisfied by the randomness which arises in sample-wise or coordinate-wise stochastic updates. In fact, it generally requires that additional, near-isotropic noise be added at each iteration, which yields convergence rates that depend heavily on problem parameters like dimension. In contrast, our results hold for the simplest implementation of gradient descent and thus do not suffer from the slow convergence associated with adding high-variance noise to each iterate. 

But is this strict saddle property reasonable? Many works have answered in the affirmative by demonstrating that many objectives of interest do in fact satisfy the ``strict saddle'' property: PCA, a fourth-order tensor factorization \citep{ge2015escaping}, formulations of dictionary learning  \citep{sun2015complete2,sun2015complete1} and phase retrieval \citep{sun2016phase}.

To obtain provable guarantees, the authors of \cite{sun2015complete2,sun2015complete1} and \cite{sun2016phase} adopt trust-region methods which leverage Hessian information in order to circumvent saddle points. This approach joins a long line of related strategies, including: a modified Newton's method with curvilinear line search \citep{more1979use}, the modified Cholesky method \citep{gill1974newton}, trust-region methods \citep{conn2000trust}, and the related cubic regularized Newton's method \citep{nesterov2006cubic}, to name a few. Specialized to deep learning applications, \cite{dauphin2014identifying,pascanu2014saddle} have introduced a saddle-free Newton method. 

Unfortunately, such curvature-based optimization algorithms have a per-iteration computational complexity which scales quadratically or even cubically in the dimension $d$, rendering them unsuitable for optimization of high dimensional functions. In contrast, the complexity of an iteration of gradient descent is linear in dimension. We also remark that the authors of \cite{sun2016phase} empirically observe gradient descent with $100$ random initializations on the phase retrieval problem reliably converges to a local minimizer, and one whose quality matches that of the solution found using more costly trust-region techniques. 

More broadly, many recent works have shown that gradient descent plus smart initialization provably converges to the global minimum for a variety of non-convex problems: such settings include matrix factorization \citep{Keshavan09,zhaononconvex} , phase retrieval \citep{candes2015phase,cai2015optimal}, dictionary learning \citep{arora2015simple}, and latent-variable models \citep{zhang2014spectral}. While our results only guarantee convergence to local minimizers, they eschew the need for complex and often computationally prohibitive initialization procedures. 

Finally, some preliminary results have shown that there are settings in which if an algorithm converges to a saddle point it necessarily has a small objective value. For example, \cite{choromanska2014loss} studies the loss surface of a particular Gaussian random field as a proxy for understanding the objective landscape of deep neural nets. The results leverage the Kac-Rice Theorem \citep{adler2009random,auffinger2013random}, and establish that that critical points with more positive eigenvalues have lower expected function value, often close to that of the global minimizer. 
We remark that functions drawn from this Gaussian random field model share the strict saddle property defined above, and so our results apply in this setting.  On the other hand, our results are considerably more general, as they do not place stringent generative assumptions on the objective function $f$.

\subsection{Organization}
The rest of the paper is organized as follows. Section \ref{sec:prelim} introduces the notation and definitions used throughout the paper. Section \ref{sec:intuition} provides an intuitive explanation for why it is unlikely that gradient descent converges to a saddle point, by studying a non-convex quadratic and emphasizing the analogy with power iteration. Section \ref{sec:main} states our main results which guarantee gradient descent converges to only local minimizers, and also establish rates of convergence depending on the local geometry of the minimizer. The primary tool we use is the local stable manifold theorem, accompanied by inversion of gradient descent via the proximal point algorithm. Finally, we conclude in Section \ref{sec:conclusion} by suggesting several directions of future work.

\section{Preliminaries}
\label{sec:prelim}
Throughout the paper, we will use $f$ to denote a real-valued function in $C^2$, the space of twice-continuously differentiable functions, and $g$ to denote the corresponding gradient map with step size $\alpha$,
\begin{align}
g(x) = x-\alpha \nabla f(x).
\end{align}
The Jacobian of $g$ is given by $Dg(x)_{ij} = \frac{\partial g_i}{\partial x_j} (x)$, or $Dg(x) = I- \alpha \nabla ^2 f(x) $. In addition to being $C^2$, our main regularity assumption on $f$ is that it has a Lipschitz gradient:
\begin{align*}
\norm{\nabla f(x) - \nabla f(y) }_2 \le L \norm{x-y}_2.
\end{align*}

The $k$-fold composition of the gradient map $g^k (x)$ corresponds to performing $k$ steps of gradient descent initialized at $x$. The iterates of gradient descent will be denoted $x_k := g^k (x_0)$. All the probability statements are with respect to $\nu$, the distribution of $x_0$, which we assume is absolutely continuous with respect to Lebesgue measure. 

A fixed point of the gradient map $g$ is a critical point of the function $f$. Critical points can be saddle points, local minima, or local maxima. In this paper, we will study the critical points of $f$ via the fixed points of $g$, and then apply dynamical systems theory to $g$.

\begin{definition}
\begin{enumerate}
\item A point $x^*$ is a critical point of $f$ if it is a fixed point of the gradient map $g(x^*)=x^*$, or equivalently $\nabla f(x^*)=0$.

\item A critical point $x^*$ is isolated if there is a neighborhood $U$ around $x^*$, and $x^*$ is the only critical point in $U$.

\item A critical point is a local minimum if there is a neighborhood $U$ around $x^*$ such that $f(x^*) \le f(x)$ for all $x \in U$, and a local maximum if $f(x^*) \ge f(x)$.

\item A critical point is a saddle point if for all neighborhoods $U$ around $x^*$, there are $x,y \in U$ such that $f(x) \le f(x^*) \le f(y)$.
\end{enumerate}
\end{definition}

As mentioned in the introduction, we will be focused on saddle points that have directions of strictly negative curvature.  This notion is made precise by the following definition.
\begin{definition}[Strict Saddle]
A critical point $x^*$ of $f$ is a strict saddle if $\lambda_{\min} (\nabla^2 f(x^*))<0$.
\label{def:strict-saddle}
\end{definition}

Since we are interested in the attraction region of a critical point, we define the stable set.
\begin{definition}[Global Stable Set]
The global stable set $W^s (x^*)$ of a critical point $x^*$ is the set of initial conditions of gradient descent that converge to $x^*$:
\begin{align*}
W^s (x^*) = \{ x: \lim_k g^k (x) = x^*\}.
\end{align*}
\end{definition}

\section{Intuition}
\label{sec:intuition}
To illustrate why gradient descent does not converge to saddle points, consider the case of a non-convex quadratic, $f(x)=\frac{1}{2} x^T H x$.  Without loss of generality, assume $H = \diag(\lambda_1,...,\lambda_n)$  with $\lambda_1, ..., \lambda_k >0$ and $\lambda_{k+1}, \dots, \lambda_n <0$.  $x^*=0$ is the unique critical point of this function and the Hessian at $x^*$ is $H$.  Note that gradient descent initialized from $x_0$ has iterates
\[
	x_{k+1} = \sum_{i=1}^n (1-\alpha \lambda_i)^{k+1} \langle e_i,x_0 \rangle e_i\,.
\]
where $e_i$ denote the standard basis vectors. This iteration resembles power iteration with the matrix $I - \alpha H$. 

The gradient method is guaranteed to converge with a constant step size provided $0< \alpha < \frac2L$~\citep{nesterov2004introductory}.  For this quadratic $f$, $L$ is equal to $\max |\lambda_i|$.  Suppose $\alpha < 1/L$, a slightly stronger condition.  Then we will have $(1- \alpha \lambda_i)< 1$ for $i \le k$ and $(1-\alpha \lambda_i) >1$ for $i >k$. If  $x_0 \in E_s:=\linspan(e_1,\ldots, e_k)$, then $x_k$ converges to the saddle point at $0$ since $(1-\alpha \lambda_i)^{k+1} \to 0$. However, if $x_0$ has a component outside $E_s$ then gradient descent diverges to $\infty$.  For this simple quadratic function, we see that the global stable set (attractive set) of $0$ is the subspace $E_s$.  Now, if we choose our initial point at random, the probability of that point landing in $E_s$ is zero.

As an example of this phenomena for a non-quadratic function, consider the following example from \cite[Section 1.2.3]{nesterov2004introductory}. Letting $f(x,y) = \frac12 x^2 +\frac14 y^4-\frac12 y^2$, the corresponding gradient mapping is 
\begin{align*}
g(x) &= \begin{bmatrix}
(1-\alpha) x\\
(1+\alpha) y - \alpha y^3
\end{bmatrix}.
\end{align*}
The critical points are
\begin{align*}
z_1 = \begin{bmatrix} 0 \\0 \end{bmatrix}, \quad z_2 = \begin{bmatrix} 0 \\ -1 \end{bmatrix},\quad
z_3 = \begin{bmatrix} 0 \\ 1 \end{bmatrix}.
\end{align*}
 The points $z_2$ and $z_3$ are isolated local minima, and $z_1$ is a saddle point.
 
Gradient descent initialized from any point of the form $\begin{bmatrix} x \\ 0 \end{bmatrix}$ converges to the saddle point $z_1$. Any other initial point either diverges, or converges to a local minimum, so the stable set of $z_1$ is the $x$-axis, which is a zero measure set in $\reals^2$. By computing the Hessian, 
\begin{align*}
\nabla ^2 f(x) = \begin{bmatrix}
1 & 0 \\
0 & 3y^2 -1
\end{bmatrix}
\end{align*}
we find that $\nabla ^2 f(z_1)$ has one positive eigenvalue with eigenvector that spans the $x$-axis, thus agreeing with our above characterization of the stable set. If the initial point is chosen randomly, there is zero probability of initializing on the $x$-axis and thus zero probability of converging to the saddle point $z_1$.
 
In the general case, the local stable set $W^{s}_{loc} (x^*)$ of a critical point $x^*$ is well-approximated by the span of the eigenvectors corresponding to positive eigenvalues.  By an application of Taylor's theorem, one can see that if the initial point $x_0$ is uniformly random in a small neighborhood around $x^*$, then the probability of initializing in the span of these eigenvectors is zero whenever there is a negative eigenvalue.  Thus, gradient descent initialized at $x_0$ will leave the neighborhood. The primary difficulty is that $x_0$ is randomly distributed over the entire domain, not a small neighborhood around $x^*$, and Taylor's theorem does not provide any global guarantees.

However, the global stable set can be found by inverting the gradient map via $g^{-1}$. Indeed, the global stable set is precisely $\cup_{k=0}^\infty g^{-k} ( W^{s}_{loc} (x^*))$. This follows because if a point $x$ converges to $x^*$, then for some sufficiently large  $k$ it must enter the local stable set.  That is, $x$ converges to $x^*$ if and only if $g^k (x) \in W^{s}_{loc}$ for sufficiently large $k$. If $W^{s}_{loc} (x^*)$ is of measure zero, then $g^{-k} ( W^{s}_{loc} (x^*))$ is also of measure zero, and hence the global stable set is of measure zero. Thus, gradient descent will never converge to $x^*$ from a random initialization.

In Section \ref{sec:main}, we formalize the above arguments by showing the existence of an inverse gradient map. The case of degenerate critical points, critical points with zero eigenvalues, is more delicate; the geometry of the global stable set is no longer characterized by only the number of positive eigenvectors. However in Section \ref{sec:main}, we show that if a critical point has at least one negative eigenvalue, then the global stable set is of measure zero.

 \section{Main Results}
\label{sec:main}

We now state and prove our main theorem, making our intuition rigorous.
\begin{theorem}
Let $f$ be a $C^2$ function  and $x^*$ be a strict saddle. Assume that $0<\alpha <\frac1L$, then
$$
\Pr( \lim_k x_k = x^* ) =0.
$$
\label{thm:main}
\end{theorem}

That is, the gradient method never converges to saddle points, provided the step size is not chosen aggressively.  Greedy methods that use precise line search may still get stuck at stationary points.  However, a short-step gradient method will only converge to minimizers. 

\begin{remark} Note that even for the convex functions method, a constant step size slightly less than $1/L$ is a nearly optimal choice.  Indeed, for $\theta<1$, if one runs the gradient method with step size of $\theta/L$ on a convex function a convergence rate of $O(\frac{1}{\theta T})$ is attained.  
\end{remark}

\begin{remark}
When $\lim_k x_k$ does not exist, the above theorem is trivially true.
\end{remark}

To prove Theorem~\ref{thm:main},  our primary tool will be the theory of Invariant Manifolds. Specifically, we will use  Stable-Center Manifold theorem developed in \cite{smale1967differentiable,shub1987global,hirsch1977invariant}, which allows for a local characterization of the stable set.  Recall that a map $g: X \to Y$ is a diffeomorphism if $g$ is a bijection, and $g$ and $g^{-1}$ are continuously differentiable.

\begin{theorem}[Theorem III.7, \cite{shub1987global}]
Let $0$ be a fixed point for the $C^r$ local diffeomorphism $\phi: U \to E$, where $U$ is a neighborhood of $0$ in the Banach space $E$. Suppose that $E = E_s \oplus E_u$, where $E_s$ is the span of the eigenvectors corresponding to eigenvalues less than or equal to $1$ of $D \phi(0)$, and $E_u$ is the span of the eigenvectors corresponding to eigenvalues greater than $1$ of $D \phi(0)$. Then there exists a $C^r$ embedded disk $W^{cs}_{loc}$ that is tangent to $E_s$ at $0$ called the \emph{local stable center manifold}.  Moreover, there exists a neighborhood $B$ of $0$, such that $\phi(W^{cs}_{loc} ) \cap B \subset W^{cs} _{loc}$, and $\cap_{k=0}^\infty \phi^{-k} (B) \subset W^{cs}_{loc}$.
\label{thm:stable-center-manifold}
\end{theorem}

To unpack all of this terminology, what the stable manifold theorem says is that if there is a map that diffeomorphically deforms a neighborhood of a critical point, then this implies the existence of a local stable center manifold $W^{cs}_{loc}$ containing the critical point.  This manifold has dimension equal to the number of eigenvalues of the Jacobian of the critical point that are less than $1$.   $W^{sc}_{loc}$ contains all points that are locally forward non-escaping meaning, in a smaller neighborhood $B$, a point converges to $x^*$ after iterating $\phi$ only if it is in $W^{cs}_{loc} \cap B$.

Relating this back to the gradient method, replace $\phi$ with our gradient map $g$ and let $x^*$ be a strict saddle point.  We first record a very useful fact:
\begin{proposition}
The gradient mapping $g$ with step size $\alpha <\frac1L$ is a diffeomorphism.
\label{thm:gd-diffeo} 
\end{proposition}
We will prove this proposition below.  But let us first continue to apply the stable manifold theorem.  Note that $Dg(x) = I-\alpha \nabla^2 f(x)$.   Thus, the set $W^{cs}_{loc}$ is a manifold of dimension equal to the number of non-negative eigenvalues of the $\nabla^2 f(x)$.  Note that by the strict saddle assumption, this manifold has strictly positive codimension and hence has measure zero.

Let $B$ be the neighborhood of $x^*$ promised by the Stable Manifold Theorem.  If $x$ converges to $x^*$ under the gradient map, then there exists a $T$ such that $g^t(x) \in B$ for all $t\geq T$.  This means that $g^t(x) \in \cap_{k=0}^\infty g^{-k} (B)$, and hence, $g^t(x) \in W^{cs}_{loc}$.  That is, we have shown that
$$
W^{s}(x^*) \subseteq \bigcup_{l\ge 0 }^{\infty} g^{-l} \left( \bigcap_{k=0}^\infty g^{-k} (B) \right).
$$
Since diffeomorphisms map sets of measure zero to sets of measure zero, and countable unions of measure zero sets have measure zero, we conclude that $W^{s}$ has measure zero.  That is, we have proven Theorem~\ref{thm:main}.

\subsection{Proof of Proposition~\ref{thm:gd-diffeo}}

We first check that $g$ is injective from $\reals^n \to \reals^n$ for $\alpha < \frac{1}{L }$.  Suppose that there exist $x$ and $y$ such that $g(x) =g(y)$.
Then we would have $x-y = \alpha(  \nabla f(x) -\nabla f(y))$ and hence
\begin{align*}
\|x-y\| = \alpha\| \nabla f(x) -\nabla f(y)\| \leq \alpha L \|x -y\|\,.
\end{align*}
Since $\alpha L<1$, this means $x=y$.

To show the gradient map is surjective, we will construct an explicit inverse function. The inverse of the gradient mapping is given by performing the proximal point algorithm on the function $-f$. The proximal point mapping of $-f$ centered at $y$ is given by
\begin{align*}
x_y= \arg\min_{x} \frac12 \norm{x-y}^2 - \alpha f(x).
\end{align*}
For $\alpha< \frac1L$, the function above is strongly convex with respect to $x$, so there is a unique minimizer.
Let $x_y$ be the unique minimizer, then by KKT conditions,
\begin{align*}
y&=x_y-\nabla f(x_y) = g(x_y)\,.
\end{align*}
Hence, $x_y$ is mapped to $y$ by the gradient map.

We have already shown that $g$ is a bijection, and continuously differentiable.  Since $Dg(x) =I- \alpha \nabla^2 f(x)$ is invertible for $\alpha< \frac1L$, the inverse function theorem guarantees $g^{-1} $ is continuously differentiable, completing the proof that $g$ is a diffeomorphism.

\subsection{Further consequences of Theorem~\ref{thm:main}}

\begin{corollary}
Let $C$ be the set of saddle points and assume they are all strict. If  $C$ has at most countably infinite cardinality, then
$$
\Pr( \lim_k x_k \in C ) =0.
$$
\label{thm:no-saddles}
\end{corollary}
\begin{proof}
By applying Corollary \ref{thm:main} to each point $x^* \in C$, we have that $\Pr( \lim_k x_k = x^*)=0$. Since the critical points are countable, the conclusion follows since countable union of null sets is a null set. 
\end{proof}
\begin{remark}
If the saddle points are isolated points, then the set of saddle points is at most countably infinite.
\end{remark}

\begin{theorem}
Assume the same conditions as Theorem \ref{thm:no-saddles} and $\lim_k x_k $ exists, thien $ \Pr(\lim_k x_k = x^\star) =1$, where $x^\star$ is a local minimizer.
\label{thm:local-min}
\end{theorem}

\begin{proof}
Using the previous theorem, $ \Pr(\lim_k x_k \in  C) =0$. Since $\lim_k x_k$ exists and there is zero probability of converging to a saddle, then $\Pr( \lim_k x_k = x^*)=1$, where $x^*$ is a local minimizer.
\end{proof}

We now discuss two  sufficient conditions for $\lim_k x_k$ to exist. The following proposition prevents $x_k$ from escaping to $\infty$, by enforcing that $f$ has compact sublevel sets, $\{x: f(x) \le c\}$. This is true for any coercive function, $\lim_{\norm{x} \to \infty} f(x) =\infty$, which holds in most machine learning applications since $f$ is usually a loss function.
\begin{proposition}[Proposition 12.4.4 of \cite{lange2013optimization}]
Assume that $f$ is continuously differentiable, has isolated critical points, and compact sublevel sets, then $\lim_k x_k$ exists and that limit is a critical point of $f$.
\end{proposition}

The second sufficient condition for $\lim_k x_k$ to exist is based on the Lojasiewicz gradient inequality, which characterizes the steepness of the gradient near a critical point. The Lojasiewicz inequality ensures that the length traveled by the iterates of gradient descent is finite. This will also allow us to derive rates of convergence to a local minimum.
\begin{definition}[Lojasiewicz Gradient Inequality]
A critical point $x^*$ is satisfies the Lojasiewicz gradient inequality if there exists a neighborhood $U$, $m, \epsilon>0$, and $0\le a <1$ such that
\begin{equation}
\norm{\nabla f(x)} \ge m | f(x) -f(x^*)|^{a}
\label{eq:gradient-ineq}
\end{equation}
for all x in $\{x\in U: f(x^*) < f(x) < f(x^*) +\epsilon\}$.
\label{def:gradineq}
\end{definition}
The Lojasiewicz inequality is very general as discussed in \cite{bolte2010characterizations,attouch2010proximal,attouch2013convergence}. In fact every analytic function satisfies the Lojasiewicz inequality. Also if the solution is $\mu$-strongly convex in a neighborhood, then the Lojasiewicz inequality is satisfied with parameters $a= \frac12$, and $m= \sqrt{2\mu}$.

\begin{proposition}
Assume the same conditions as Theorem \ref{thm:no-saddles}, and the iterates do not escape to $\infty$, meaning $\{x_k\}$ is a bounded sequence. Then $\lim_k x_k$ exists and $\lim_k x_k =x^*$ for a local minimum $x^*$.

Furthermore if $x^*$ satisfies the Lojasiewicz gradient inequality for 
$0< a \le\frac12$, then for some $C$ and $b<1$ independent of $k$,
\begin{align*}
\norm{x_k - x^*} \le C b^k.
\end{align*}
For $\frac12<a<1$,
\begin{align*}
\norm{x_k - x^*} \le \frac{C}{k^{(1-a)/(2a-1)}}.
\end{align*}
\label{prop:rates}
\end{proposition}

\begin{proof}
The first part of the theorem follows from \cite{absil2005convergence}, which shows that $\lim_k x_k$ exists. By Theorem \ref{thm:local-min}, $\lim_k x_k $ is a local minimizer $x^*$. Without loss of generality, we may assume that $f(x^*)=0$ by shifting the function.

\cite{absil2005convergence} also establish
\begin{align*}
\sum_{j=k}^\infty\norm{x_{j+1} -x_j}\le  \frac{2}{\alpha m (1-a)} f(x_k)^{1-a}.
\end{align*}
Define $e_k = \sum_{j=k}^\infty\norm{x_{j+1} -x_j}$, and since $e_k \ge \norm{x_{k} -x^*}$ it suffices to upper bound $e_k$.

Since we have established that $x_k$ converges, for $k$ large enough we can use the gradient inequality and $\nabla f(x_k) = \frac{x_{k} -x_{k+1}}{\alpha}$:
\begin{align*}
e_k &\le \frac{2}{\alpha m (1-a) } f(x_k)^{1-a}\\
&\le \frac{2}{\alpha m^{1/a} (1-a) } \norm{\nabla f(x_k)}^{(1-a)/a} \\
&\le \frac{2}{(m\alpha)^{1/a} (1-a) }( e_k - e_{k+1} )^{(1-a)/a}.
\end{align*}
Define $\beta = \frac{2}{(m\alpha)^{1/a} (1-a) }$ and $d=\frac{a}{1-a}$.
First consider the case $0\le a \le \frac12$, then $d \le 1$.
Thus,
\begin{align*}
e_k &\le \beta (e_k - e_{k+1})^{1/d}\\
e_{k+1} &\le e_k - \left(\frac{e_k}{\beta}\right)^d\\
&\le \left( 1- \frac{1}{\beta^d} \right) e_k,
\end{align*}
where the last inequality uses $e_k<1$ and $d \le 1$.

For $\frac12 < a <1$, we have established $e_{k+1} \le e_k - \frac{1}{\beta^d} e_k ^d$. We show by induction that $e_{k+1} \le \frac{C}{(k+1)^{(1-a)/(2a-1)}}$.
The inductive hypothesis guarantees us $e_k \le \frac{C}{k^{(1-a)/(2a-1)}}$, so
\begin{align*}
e_{k+1}&\le \frac{C}{k^{(1-a)/(2a-1)}}-\frac{C^d/\beta^d}{k^{a/(2a-1)}}\\
&= \frac{C k - C^d/ \beta^d}{k \cdot k^{(1-a)/(2a-1)}}\\
&\le \frac{C (k -C^{d-1}/\beta^d) }{(k-1) (k+1)^{(1-a)/(2a-1)}}.
\end{align*}
For $C^{d-1}>\beta^d$,we have shown $e_{k+1} \le \frac{C}{ (k+1)^{(1-a)/(2a-1)}}$.
\end{proof}

\section{Conclusion}
\label{sec:conclusion}
We have shown that gradient descent with random initialization and appropriate constant step size does not converge to a saddle point. Our analysis relies on a characterization of the local stable set from the theory of invariant manifolds. The geometric characterization is not specific to the gradient descent algorithm.  To use Theorem \ref{thm:main}, we simply need the update step of the algorithm to be a diffeomorphism. For example if $g$ is the mapping induced by the proximal point algorithm, then $g$ is a diffeomorphism with inverse given by gradient ascent on $-f$. Thus the results in Section \ref{sec:main} also apply to the proximal point algorithm.  That is, \emph{the proximal point algorithm does not converge to saddles}. We expect that similar arguments can be used to show ADMM, mirror descent and coordinate descent do not converge to saddle points under appropriate choices of step size.  Indeed, convergence to minimizers has been empirically observed for the ADMM algorithm~\cite{sun2015complete1}.

It is not clear if the step size restriction ($\alpha<1/L$) is necessary to avoid saddle points.  Most of the constructions where the gradient method converges to saddle points require fragile initial conditions as discussed in Section~\ref{sec:intuition}.  It remains a possibility that methods that choose step sizes greedily, by Wolfe Line Search or backtracking, may still avoid saddle points provided the initial point is chosen at random.  We leave such investigations for future work.

Another important piece of future work would be relaxing the conditions on isolated saddle points.  It is possible that for the structured problems that arise in machine learning, whether in matrix factorization or convolutional neural networks, that saddle points are isolated after taking a quotient with respect to the associated symmetry group of the problem.  Techniques from dynamical systems on manifolds may be applicable to understand the behavior of optimization algorithms on problems with a high degree of symmetry.

It is also important to understand how stringent the strict saddle assumption is. Will a perturbation of a function always satisfy the strict saddle property? \cite{adler2009random} provide very general sufficient conditions for a random function to be Morse, meaning the eigenvalues at critical points are non-zero, which implies the strict saddle condition. These conditions rely on checking the density of $\nabla^2 f(x)$ has full support conditioned on the event that $\nabla f(x)=0$. This can be explicitly verified for functions $f$ that arise from learning problems.

However, we note that there are very difficult unconstrained optimization problems where the strict saddle condition fails.  Perhaps the simplest is optimization of quartic polynomials.  Indeed, checking if $0$ is a local minimizer of the quartic
\[
	f(x) = \sum_{i,j=1}^n q_{ij} x_i^2 x_j^2
\]
is equivalent to checking whether the matrix $Q=[q_{ij}]$ is co-positive, a co-NP complete problem.  For this $f$, the Hessian at $x=0$ is zero.  Interestingly, the strict saddle property failing is analogous in dynamical systems to the existence of a \emph{slow manifold} where complex dynamics may emerge.  Slow manifolds give rise to metastability, bifurcation, and other chaotic dynamics, and it would be intriguing to see how the analysis of chaotic systems could be applied to understand the behavior of optimization algorithms around these difficult critical points.

\section*{Acknowledgements}
The authors would like to thank Chi Jin, Tengyu Ma, Robert Nishihara, Mahdi Soltanolkotabi, Yuekai Sun, Jonathan Taylor, and Yuchen Zhang for their insightful feedback.  MS is generously supported by an NSF Graduate Research Fellowship.  BR is generously supported by ONR awards  N00014-14-1-0024, N00014-15-1-2620, and N00014-13-1-0129, and NSF awards CCF-1148243 and CCF-1217058.  MIJ is generously supported by  ONR award N00014-11-1-0688 and by the ARL and the ARO under grant number W911NF-11-1-0391.  This research is supported in part by NSF CISE Expeditions Award CCF-1139158, DOE Award SN10040 DE-SC0012463, and DARPA XData Award FA8750-12-2-0331, and gifts from Amazon Web Services, Google, IBM, SAP, The Thomas and Stacey Siebel Foundation, Adatao, Adobe, Apple Inc., Blue Goji, Bosch, Cisco, Cray, Cloudera, Ericsson, Facebook, Fujitsu, Guavus, HP, Huawei, Intel, Microsoft, Pivotal, Samsung, Schlumberger, Splunk, State Farm, Virdata and VMware.

\begin{small}
\bibliographystyle{plain}
\bibliography{gradient}
\end{small}

\end{document}